\documentclass[10pt,journal,cspaper,compsoc]{IEEEtran}
\ifCLASSOPTIONcompsoc
  \usepackage[nocompress]{cite}
\else
\fi

\usepackage[cmex10]{amsmath}
\usepackage{algorithmic}

\usepackage{array}

\usepackage{url}

\usepackage{amssymb}
\usepackage{amsthm}
\usepackage{algorithm}
\usepackage{multirow}
\usepackage{graphicx}
\usepackage{epsfig,epsf,subfigure}
\usepackage{epstopdf}

\newtheorem{defi}{\textbf{Definition}~}

\newtheorem{propos}{\textbf{Proposition}~}

\newcommand{\eg}{{\it e.g.}}
\newcommand{\ie}{{\it i.e.}}

\begin{document}
%
\title{ESSP: An Efficient Approach to Minimizing Dense and Nonsubmodular Energy Functions}
%
%
%
%

\author{Wei~Feng,~\IEEEmembership{Member,~IEEE,}
        Jiaya~Jia,~\IEEEmembership{Senior Member,~IEEE,}
        and~Zhi-Qiang~Liu,
\IEEEcompsocitemizethanks{\IEEEcompsocthanksitem W. Feng is with School of Computer Science and Technology, Tianjin University, China.
E-mail: wfeng@ieee.org (Correspondence author). 
\IEEEcompsocthanksitem J. Jia is with Department of Computer Science and Engineering, The Chinese University of Hong Kong, Hong Kong, China.\protect\\Z.-Q. Liu is with School of Creative Media, City University of Hong Kong, Hong Kong, China.}
\thanks{}}

\IEEEcompsoctitleabstractindextext{%
\begin{abstract}
Many recent advances in computer vision have demonstrated the impressive power of dense and nonsubmodular energy functions in solving visual labeling problems. However, minimizing such energies is challenging. None of existing techniques (such as s-t graph cut, QPBO, BP and TRW-S) can individually do this well. In this paper, we present an efficient method, namely ESSP, to optimize binary MRFs with arbitrary pairwise potentials, which could be nonsubmodular and with dense connectivity. We also provide a comparative study of our approach and several recent promising methods. From our study, we make some reasonable recommendations of combining existing methods that perform the best in different situations for this challenging problem. Experimental results validate that for dense and nonsubmodular energy functions, the proposed approach can usually obtain lower energies than the best combination of other techniques using comparably reasonable time.
\end{abstract}

}

\maketitle

\IEEEdisplaynotcompsoctitleabstractindextext

%
\IEEEpeerreviewmaketitle

\section{Introduction}
\label{sec:intro}

Algorithms for discrete energy minimization play a fundamental role in computer vision and image analysis. Many early vision problems can be formulated as minimizing an energy function of the following form \begin{equation}\label{eq:mrf}
E(\mathbf{x}) = \theta_{const} + \sum_{u \in \mathcal{N}} \theta_u (x_u)  + \sum_{(u,v) \in \mathcal{M}} \theta_{uv} (x_u x_v) .
\end{equation}
where $\mathcal{N}$ is the set of pixels, $x_u \in \mathcal{L}$ denotes the label of pixel $u$, and neighborhood system $\mathcal{M}$ defines the pairwise dependency between pixels. Sets $\mathcal{N}$ and $\mathcal{M}$ jointly compose a graph, on which the energy function Eq.~(\ref{eq:mrf}) is defined. The exact meaning of label space $\mathcal{L}=\{0,1,\cdots,L\}$ depends on the specific problems, \eg, in image segmentation the labels are segment indices, while for stereo they represent disparities. In this paper, we focus on binary labeling problems, \ie, we consider only binary label set $\mathcal{L}=\{0,1\}$. For most early vision problems, the energy function of Eq.~(\ref{eq:mrf}) is derived from the ubiquitous MRF model \cite{em_ref:mrfcomp08}. 

Last decade has witnessed the great success of efficient energy minimization techniques in computer vision. Promising algorithms include graph cuts and QPBO \cite{em_ref:boyk01a,em_ref:kolm04,em_ref:kolm07,em_ref:roth07}, belief propagation (BP) \cite{em_ref:weiss01,em_ref:yedidia00,em_ref:ebpijcv06}, tree-reweighted message passing (TRW) \cite{em_ref:wain05,em_ref:kolmogorov06}, linear programming \cite{em_ref:werner07}, and convexity-related methods \cite{em_ref:ishikawa03,em_ref:kumar08,em_ref:kumar07}. These methods have triggered a significant progress in the state-of-the-art of many early vision problems, such as image segmentation \cite{seg_ref:svlmrf10,ia_ref:feng08}, stereo \cite{em_ref:boyk01a,em_ref:tappen03} and restoration \cite{em_ref:boyk01a,em_ref:kolm07}. Recently, the performance of these techniques in minimizing energy functions defined on $4$-connected grid graphs has been extensively studied \cite{em_ref:mrfcomp08,em_ref:tappen03,em_ref:meltzer05}. These studies consistently conclude that for some visual labeling problems, such as stereo, existing methods are able to obtain a near global minimum solution efficiently. However, it has also been shown that the optimal labeling proposed by global minimum energy has larger error statistics than other suboptimal solutions for benchmark stereo pairs \cite{em_ref:meltzer05}. This indeed reflects the deficiency of the energy model itself, and suggests that further improvements can only be achieved by using more complicated and powerful models.

Most recently, some sophisticated models, such as random fields with higher connectivity \cite{em_ref:kolm06,FoE:RothFoE} and with higher-order cliques \cite{em_ref:rkfj09,em_ref:woodford08,FoE:LanHighOrderBP,em_ref:ormrf12}, have shown their impressive potentials in handling difficult visual labeling problems. Despite the very different formulations of these models, they usually rely on minimizing some dense and nonsubmodular energy functions \cite{em_ref:rkfj09,em_ref:gcmrf11}. However, minimizing such energy functions is challenging for existing techniques \cite{em_ref:mrfcomp08,em_ref:kolm06}.

In this paper, we present a new method to optimize general MRFs with arbitrary pairwise potentials. Since the energy function of a binary MRF is a quadratic Pseudo-boolean function (QPBF). Our algorithm can also be viewed as a general approximate solution to dense and nonsubmodular QPBF minimization \cite{em_ref:boros02}. The core of our algorithm is an extended submodular-supermodular procedure (ESSP) that expands the classical submodular-supermodular procedure (SSP) \cite{em_ref:nara05} to minimize QPBFs of any type. Specifically, we present an undirected graph characterization of QPBF that has two desirable properties: (1)~it enables automatic submodular-supermodular decomposition for a QPBF; and (2)~it transforms a QPBF to a symmetric set function, thus providing an efficient way to suppress the supermodular part of the QPBF and to apply more efficient graph cut solutions \cite{em_ref:queyranne98,em_ref:onmgc2012}. The first property extends the application scope of SSP to general QPBF minimization; and the second property enables us to control the accuracy of our optimization and to balance the accuracy and complexity.

We provide a thorough comparative study on the performance of existing methods for dense and nonsubmodular QPBF minimization. We evaluated three important factors, \ie, \emph{connectivity}, \emph{supermodularity ratio} and \emph{unary guidance}, that closely affect the hardness of the problem from different aspects. From our study, we figure out several hardness situations of the problem, and make a reasonable recommendation of combining existing methods in each situation that performs the best. Experimental results also validate that for dense and nonsubmodular energy functions, the proposed ESSP algorithm can always be used to improve the results of the best combinations of existing methods with reasonable time.

\section{Related Work}
\label{sec:rw}

Function $f(\mathbf{x}):\{0,1\}^n \mapsto \mathbb{R}$ is a QPBF if it contains only unary and pairwise items, \ie, $f(\mathbf{x}) = \sum_i u(x_i) + \sum_{i,j} p(x_i, x_j)$. QPBF is a general energy functional for many problems in computer vision and machine learning. For instance, some combinatorial problems such as graph matching can be formulated as QPBF minimization.

QPBF minimization has been studied in discrete optimization for decades \cite{em_ref:boros02}. We know that if a QPBF is \emph{submodular}, it can be exactly minimized in polynomial time. However, minimizing nonsubmodular QPBFs is NP-hard. Only approximate methods exist for this general problem. One of such methods is SSP that is designed to minimize the difference of two submodular functions \cite{em_ref:nara05}. In this paper, we show that SSP can be generalized to minimize any QPBF based on an undirected graph characterization. 

QPBF is also a general form of the energy functions of binary MRFs. State-of-the-art methods for binary MRF labeling include BP \cite{em_ref:weiss01}, TRW \cite{em_ref:wain05,em_ref:kolmogorov06} and graph cuts \cite{em_ref:kolm04}. An important conclusion is that any submodular QPBF can be exactly minimized by solving an st-mincut problem on a directed graph \cite{em_ref:kolm04}. For nonsubmodular QPBFs, the roof duality \cite{em_ref:boros02} is used to find an optimal partial labeling, which is called QPBO and two variants P and I \cite{em_ref:kolm07,em_ref:roth07}. Our algorithm is in contrast to the directed graph formulation \cite{em_ref:kolm04} and QPBO(P,I) \cite{em_ref:roth07}. First, the undirected graph characterization maps a QPBF to a symmetric set function, thus making us easily suppress the supermodular part, which is important for the optimality of our algorithm. Second, we use extended SSP to handle nonsubmodular terms. As a result, our algorithm solves a minimum cut problem of an undirected graph, which needs only half the number of nodes of QPBO(P,I).

\section{The ESSP Algorithm}
\label{sec:method}

We now introduce the proposed ESSP algorithm in detail. The major idea is to convert QPBF minimization to a minimum cut problem on an undirected graph.

\subsection{Undirected Graph Characterization}

We first consider QPBFs of the following form
\begin{equation} \label{eq:qpbfstd}
E(\mathbf{x}) = \sum_{u \in \mathcal{N}} \theta_u x_u  + \sum_{(u,v) \in \mathcal{M}} \theta_{uv} x_u x_v ,
\end{equation}
where $\mathcal{N}$ denotes the set of variables need to be optimized, $\mathcal{M}$ represents the pairwise correlations between variables. For briefness, we use $n=\| \mathcal{N} \|$ and $m = \| \mathcal{M} \|$ denote the number of variables and the number of pairwise terms in $E(\mathbf{x})$, respectively. We show how to effectively characterize Eq.~(\ref{eq:qpbfstd}) by an undirected graph. Then, we will extend the method to represent general QPBFs as defined in Eq.~(\ref{eq:mrf}).

\begin{defi}[Graph characterization]
We say a function $f(\mathbf{x})$ is characterized by a graph $\mathcal{G}_{f(\mathbf{x})}$, iff for any particular labeling of $\mathbf{x}$, the function value of $f(\mathbf{x})$ (as induced by $\mathbf{x}$) is equal to the cut value of graph $\mathcal{G}_{f(\mathbf{x})}$ plus a constant.
\end{defi}

\begin{propos}
\label{prop:graphcharact}
For any QPBF with the form of Eq.~(\ref{eq:qpbfstd}), we have the following conclusions:
\begin{enumerate}
\item Linear monomial $\alpha x_u$ can be characterized by an undirected graph $\mathcal{G}_{\alpha x_u}$ (see Fig.~\ref{fig:graph_charact}(a)).
\item Quadratic monomial $\beta x_u x_v$ can be characterized by an undirected graph $\mathcal{G}_{\beta x_u x_v}$ (see Fig.~\ref{fig:graph_charact}(b)).
\item Any QPBF $E(\mathbf{x})$ with the form of Eq.~(\ref{eq:qpbfstd}) can be characterized by an undirected graph $\mathcal{G}_{E(\mathbf{x})}$.
\end{enumerate}
\end{propos}
\begin{proof}
Referring to Fig.~\ref{fig:graph_charact}, it is easy to check the correctness of Conclusion~1 and 2. Conclusion~3 directly results from 1 and 2 due to the \emph{additivity property} of graph characterization \cite{em_ref:kolm04,em_ref:kolm07}.\footnote{Note that the undirected graph characterization has also been used in exact inference of planar graphs \cite{em_ref:effexact08}. In this paper, we use it to extend the classical SSP \cite{em_ref:nara05} to optimize QPBFs of any type.}
\end{proof}

\begin{figure}[t]
\centering
\includegraphics[width=0.66\columnwidth]{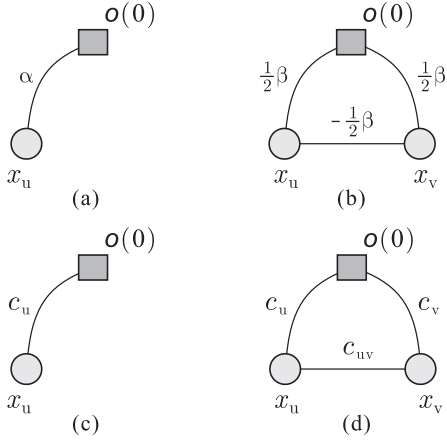}
\caption{Undirected graph characterization of QPBF: (a) $\mathcal{G}_{\alpha x_u}$ for linear monomial $\alpha x_u$; (b) $\mathcal{G}_{\beta x_u x_v}$ for quadratic monomial $\beta x_u x_v$; (c) $\mathcal{G}_{\theta_{u} (x_u)}$ for general unary function $\theta_{u} (x_u)$ with edge capacities defined in Eq.~(\ref{eq:cap1}) and (d) $\mathcal{G}_{\theta_{uv} (x_u, x_v)}$ for general quadratic function $\theta_{uv} (x_u, x_v)$ with edge capacities defined in Eq.~(\ref{eq:cap2}). Note that all graphs contain an indicator node $o$ that represents label $0$ and several variable nodes. For a cut on the graph, the nodes belonging to the same partition with $o$ are labeled as $0$, otherwise are labeled as $1$.}\label{fig:graph_charact}
\end{figure}

\begin{propos}
\label{prop:qpbfgeneral}
A general QPBF $E(\mathbf{x})$ can be characterized by an undirected graph $\mathcal{G}_{E(\mathbf{x})}$. So, $\arg\min_{\mathbf{x}} E(\mathbf{x})$ can be computed by solving a minimum cut problem on $\mathcal{G}_{E(\mathbf{x})}$.
\end{propos}
\begin{proof}
A general QPBF $E(\mathbf{x}) = \sum_{u \in \mathcal{N}} \theta_u (x_u) + \sum_{(u,v) \in \mathcal{M}} \theta_{uv} (x_u, x_v)$ can be transformed to the form of Eq.~(\ref{eq:qpbfstd}), in that $E(\mathbf{x}) = \sum_{u \in \mathcal{N}} \big[ \theta^0_u (1-x_u) + \theta^1_u x_u \big] + \sum_{(u,v) \in \mathcal{M}} \, \big[\theta^{00}_{uv} (1-x_u)(1-x_v) + \theta^{01}_{uv} (1-x_u)x_v + \theta^{10}_{uv} x_u(1-x_v) + \theta^{11}_{uv} x_ux_v \big]$. Then, the conclusion follows directly from Proposition~\ref{prop:graphcharact}. Note that, for briefness, we use $\theta^{x_u}_{u}$ to represent $\theta_{u} (x_u)$ and use $\theta^{x_u x_v}_{uv}$ to represent $\theta_{uv} (x_u, x_v)$, \eg, $\theta^{0}_{u} = \theta_{u} (x_u=0)$ and $\theta^{01}_{uv} = \theta_{uv} (x_u=0, x_v=1)$.
\end{proof}

Proposition~\ref{prop:qpbfgeneral} shows that our algorithm (as introduced in next section) is a general solution to QPBF minimization. Fig.~\ref{fig:graph_charact}(c) and (d) show graph $\mathcal{G}_{\theta_{u} (x_u)}$ and $\mathcal{G}_{\theta_{uv} (x_u, x_v)}$ that characterize general unary function $\theta_{u} (x_u)$ and general quadratic function $\theta_{uv} (x_u, x_v)$ respectively, of which the edge capacities are defined as
\begin{equation} \label{eq:cap1}
c_u = \theta^{1}_{u} - \theta^{0}_{u} ,
\end{equation}
\begin{equation} \label{eq:cap2}
\left\{\begin{array}{l} c_u = \frac{1}{2} (\theta^{10}_{uv} + \theta^{11}_{uv} - \theta^{01}_{uv} -\theta^{00}_{uv}) , \\ c_v = \frac{1}{2} (\theta^{01}_{uv} + \theta^{11}_{uv} - \theta^{00}_{uv} -\theta^{10}_{uv}) , \\ c_{uv} = \frac{1}{2} (\theta^{01}_{uv} + \theta^{10}_{uv} - \theta^{00}_{uv} -\theta^{11}_{uv}) .
\end{array} \right.
\end{equation}
Besides, if denoting the value of indicator node $o$ as $x_{o}$, the undirected graph characterization implies that any QPBF $E(\mathbf{x})$ can be converted to a symmetric set function $F(\mathbf{x},x_{o})$ that represents the cut value of $\mathcal{G}_{E(\mathbf{x})}$ \cite{em_ref:queyranne98},
\begin{equation} \label{eq:symmcutfunc}
F(\mathbf{x},x_{o}) = \sum_{(u,v) \in \{o,1,\cdots,n\}^2} c_{uv} (x_u + x_v - 2x_u x_v) ,
\end{equation}
where $(u,v)$ is an edge of $\mathcal{G}_{E(\mathbf{x})}$, $c_{uv}$ is its capacity. 

\subsection{Flipping Transformation}

For a QPBF $E(\mathbf{x})$, we construct its graph characterization $\mathcal{G}_{E(\mathbf{x})}$. Then, minimizing $E(\mathbf{x})$ is transformed to a minimum cut problem on graph $\mathcal{G}_{E(\mathbf{x})}$. We already know that $E(\mathbf{x})$ can be exactly minimized if it is submodular \cite{em_ref:boros02,em_ref:kolm04}. In practice, however, some recent powerful energy functions of real problems are usually nonsubmodular \cite{em_ref:rkfj09,em_ref:gcmrf11}. Minimizing nonsubmodular functions is generally NP-hard, thus approximate solutions are required. Before going to the details of our algorithm, we first introduce an important observation on nonsubmodular functions and an equivalent transformation on the undirected graph characterization.

\begin{figure}[t]
\centering
\includegraphics[width=0.72\columnwidth]{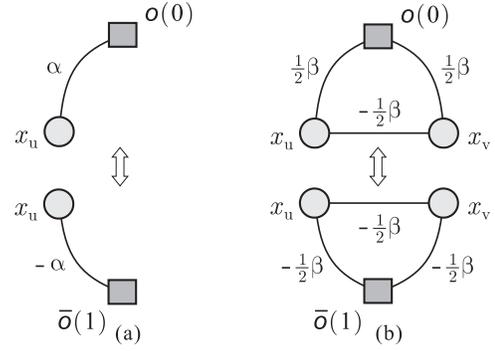}
\caption{Flipping equivalent transformation of indicator node $o$: (a) reparameterizing linear monomial; (b) reparameterizing quadratic monomial.}\label{fig:flip_trans_o}
\end{figure}

\noindent{\bf Flipping Variable.}~~For a QPBF $E(\mathbf{x})$ with $\mathbf{x}=[x_1, x_2, \cdots, x_n]^T$, flipping variable $\bar{x}_i$ is defined as $\bar{x}_i = 1-x_i$. Flipping variable has already been studied and used in the literature \cite{em_ref:boros02,em_ref:kolm07}. In this paper, we mainly use flipping variables to denote a new equivalent transformation. We observe that some nonsubmodular functions can be converted to submodular functions when flipping some original variables. For instance, $- x_1 x_2 + x_2 x_3$ is nonsubmodular, but it can be transformed to $- x_1 x_2 + x_2 (1-\bar{x}_3) = - x_1 x_2 - x_2 \bar{x}_3 + x_2$ that is a submodular function with variables $(x_1,x_2,\bar{x}_3)$. Unfortunately, not all nonsubmodular functions can be transformed to be submodular by this means, \eg, $- x_1 x_2 - x_1 x_3 + x_2 x_3$. But, we will see in next section the benefits of flipping variables in our algorithm.

\noindent{\bf Equivalent Transformation.}~~\emph{Equivalent transformation} refers to reparameterizing the original energy function $E(\mathbf{x})$ to $E'(\mathbf{x})$ that does not change the optimality of $E(\mathbf{x})$ \cite{em_ref:kolm07}. That is, for any two labelings $\mathbf{x}_1$ and $\mathbf{x}_2$, $E'(\mathbf{x}_1) \leq E'(\mathbf{x}_2) \Leftrightarrow E(\mathbf{x}_1) \leq E(\mathbf{x}_2)$. Since there are two kinds of nodes, \ie, indicator node $o$ and variable node $x_u$, in the undirected graph characterization, we have two types of flipping equivalent transformations.

The first type is about flipping the indicator node. As shown in Fig.~\ref{fig:flip_trans_o}, we can switch the meaning of indicator node from $o$ (denoting label $0$) to $\bar{o}$ (denoting label $1$) by negativing the capacities of all edges of $o$. This operation can also be conducted on only a subset of variable nodes. That is, we can use two indicator nodes, $o$ and $\bar{o}$, in graph $\mathcal{G}_{f(\mathbf{x})}$ to characterize $f(\mathbf{x})$. Generally, for an undirected graph $\mathcal{G}_{f(\mathbf{x})}$ characterizing function $f(\mathbf{x})$, flipping $o$ to $\bar{o}$ (or $\bar{o}$ to $o$) for variable nodes $x_i \in \mathcal{X}$ makes the graph exactly represent function $f(\mathbf{x})-\mu$, where $\mu$ is the sum of edges capacities linking indicator node to variable nodes in $\mathcal{X}$, where $\mathcal{X}$ is the set of all variables in $E(\mathbf{x})$.

\begin{figure}[t]
\centering
\includegraphics[width=0.9\columnwidth]{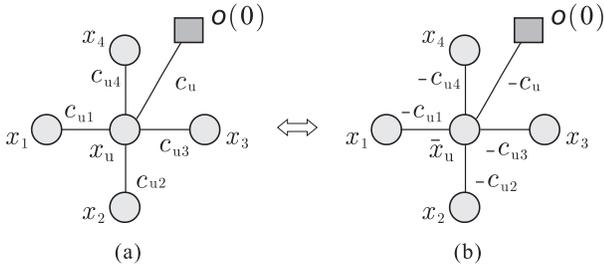}
\caption{Flipping equivalent transformation of variable node $x_u$ in (a) to $\bar{x}_u$ in (b) by negativing the capacities of all edges linking $x_u$ to other nodes, including variable nodes $x_i$ and indicator node $o$ (or $\bar{o}$), in the graph.}\label{fig:flip_trans_x}
\end{figure}

The second type is about flipping original variable $x_u$ to $\bar{x}_u$. As shown in Fig.~\ref{fig:flip_trans_x}, this can also be done by negativing the capacities of all edges of $x_u$. Flipping $x_u$ to $\bar{x}_u$ can be viewed as locally switching the meaning of label $0$ and label $1$. Specifically, for a QPBF $\theta_{uv}(x_u, x_v)$, its graph characterization can be constructed by setting the edge capacities as Eq.~(\ref{eq:cap2}). If flipping $x_u$ to $\bar{x}_u$, then: (1) indicator node $o$ become $\bar{o}$ with the capacity $c_u$ for $\bar{x}_u$, which can be reparameterized to $o$ with capacity $-c_u$; (2) the capacity of edge $(\bar{x}_u, x_v)$ is changed to $\frac{1}{2} (\theta^{00}_{uv} + \theta^{11}_{uv} - \theta^{01}_{uv} -\theta^{10}_{uv}) = -c_{uv}$.

We will show in next that \emph{flipping transformation} can be used to flip original variables in the energy function $E(\mathbf{x})$ to make the submodular part of $E(\mathbf{x})$ as large as possible, which is important for the optimality of our algorithm for QPBF minimization.

\subsection{Submodular-Supermodular Decomposition}

General QPBF minimization is challenging due to the existence of nonsubmodularity. But, there exist effective approximate solutions for some particular form of nonsubmodular functions. SSP is one of such methods and is designed to minimize the sum of a submodular function and a supermodular function \cite{em_ref:nara05}. We now extend SSP to minimize general QPBFs based on the undirected graph characterization, which provides us an convenient way to decompose any QPBF $E(\mathbf{x})$ to the sum of a submodular part $sub(E(\mathbf{x}))$ and a supermodular part $sup(E(\mathbf{x}))$. It is known that the cut function of an undirected graph is \emph{symmetric}, and \emph{submodular} if all edge capacities are nonnegative. This means that for a QPBF $E(\mathbf{x})$ we can automatically decompose $E(\mathbf{x})$ to $sub(E(\mathbf{x})) + sup(E(\mathbf{x}))$ by simply checking the positivity of edge capacities in $\mathcal{G}_{E(\mathbf{x})}$.


\begin{figure}[t]
\centering
\includegraphics[width=0.86\columnwidth]{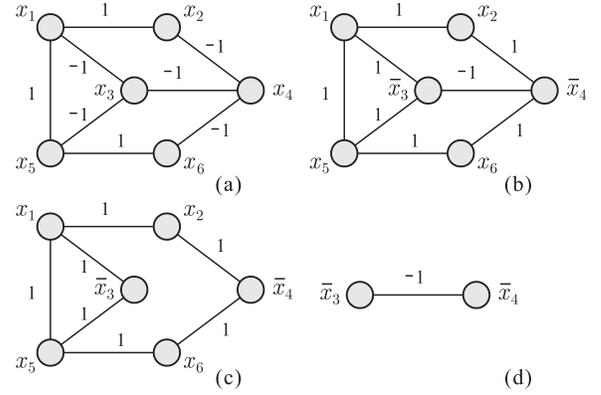}
\caption{An example of supermodular suppression for a graph with only $-1/1$ edge capacities: (a) graph $\mathcal{G}_f$ (with indicator node $o$ omitted); (b) graph $\mathcal{G}_{\bar{f}}$ after flipping $x_3$ and $x_4$, which can be decomposed to a submodular graph $\mathcal{G}_{sub(\bar{f})}$ (c) and a supermodular graph $\mathcal{G}_{sup(\bar{f})}$ (d).} \label{fig:supsup}
\end{figure}

\subsection{Supermodular Suppression}

The optimality and efficacy of SSP is highly dependent on the modular approximation of the supermodular function \cite{em_ref:nara05}. Therefore, it would be  desirable it we could suppress the influence of supermodular part $sup(E(\mathbf{x}))$ as much as possible when minimizing a QPBF $E(\mathbf{x})$ in our process. 

After submodular-supermodular decomposition, the influence of supermodular function $sup(E(\mathbf{x}))$ can be naturally measured by the sum of all negative edge capacities in $\mathcal{G}_{E(\mathbf{x})}$. Consequently, according to the definition of flipping transformation (see Fig.~\ref{fig:flip_trans_o} and \ref{fig:flip_trans_x}), to minimize the influence of $sup(E(\mathbf{x}))$ equals to minimize the following energy function $E'(\mathbf{y}) = \mathbf{y}^T \mathbf{\bar C} \mathbf{y}$, where $\mathbf{y} = [y_1, \cdots, y_n]^T$, $y_i \in \{-1, 1\}$ indicates whether flipping the $i$th variable ($y_i = -1$) or not ($y_i=1$) in supermodular suppression, $\mathbf{\bar C} = -\mathbf{C}$ and $\mathbf{C}$ is the edge capacity matrix of graph $\mathcal{G}_{E(\mathbf{x})}$. Clearly, minimizing $E'(\mathbf{y})$ is equivalent to minimize $E''(\mathbf{z}) = \mathbf{z}^T \mathbf{\bar C} \mathbf{z} + \mathbf{z}^T \mathbf{\bar C} \mathbf{1}$ with $\mathbf{z} = [z_1, \cdots, z_n]^T$ and $z_i = \frac{y_i+1}{2} \in \{0,1\}$, which is another dense and nonsubmodular energy function.

For simplicity and efficiency, instead of optimizing $E''(\mathbf{z})$ that is equally difficult as minimizing the original QPBF $E(\mathbf{x})$, in this paper, we present an efficient greedy process to suppress the supermodular part of $\mathcal{G}_{E(\mathbf{x})}$. Note that, we need only consider the edges linking two variable nodes, \ie, the \emph{variable edges}, in supermodular suppression. We first derive a descending order $\pi$ of all variable nodes according to the ratio of $\frac{\Sigma\mathrm{neg}}{\Sigma\mathrm{all}}$, where $\Sigma\mathrm{neg}$ and $\Sigma\mathrm{all}$ are the sum of absolute values of all negative capacities and all edge capacities, for each variable node, respectively. We then maintain a list $\mathbf{t}$ that records the variable nodes and their ratios of negative variable edges in the order of $\pi$. Then, we sequentially check node $x_{\pi_i}$ in $\mathbf{t}$. If the ratio is larger than $0.5$, we flip $x_{\pi_i}$ to $\bar{x}_{\pi_i}$ and update $\mathbf{t}$ accordingly. We repeat this process until no variable node need to be flipped. This can be done for any undirected graph, since in each iteration we decrease the overall influence of negative variable edges in $\mathcal{G}_{E(\mathbf{x})}$, and stop when no variable node can be flipped. Fig.~\ref{fig:supsup} shows an example of this process. Note that, there may exist multiple ways for supermodular suppression using different orders. However, the complexity of exhaustive search for the optimal flipping sequence is $2^n$. Besides, note that the above greedy process guarantees, in linear time, the influence of supermodular part $sup(E(\mathbf{x}))$ is smaller than that of submodular part $sub(E(\mathbf{x}))$. We also tested to use QPBO(P) to suppress supermodular part, since labeled nodes are guaranteed to be globally optimal. But, we found that it could barely produce better labeling. This reflects that for QPBFs with low supermodular ratios, ESSP with the proposed greedy suppression works as well as QPBO(P). Since for larger supermodular ratios few nodes will be labeled by QPBO(P), QPBO(P) cannot provide better suppression than the proposed greedy process.

In the above, we consider only negative variable edges in $\mathcal{G}_{E(\mathbf{x})}$. For negative edges $(o,x_u)$ linking indicator node $o$ and variable nodes $x_u$, which is called \emph{indicate edges}, we can equivalently replace them by edges $(\bar{o},x_u)$ with positive capacities. As a result, after supermodular suppression, graph $\mathcal{G}_{E(\mathbf{x})}$ contains two indicator nodes $o$ and $\bar{o}$.

\subsection{The Algorithm}

We now present our algorithm for general QPBF minimization in Algorithm~1. \if0\ref{alg:essp},\fi The proposed ESSP algorithm is an iterative refinement process. The initial labeling $\mathbf{x}^{(0)}$ can be the output of other methods or randomly generated. The performance of different initialization is compared in the experiments section. Each iteration of the ESSP algorithm consists of three major steps: (1) modular approximation,\footnote{Refer to \cite{em_ref:nara05} for more details about $\mathrm{modApproximation}(\cdot)$ in Algorithm~1. \if 0 \ref{alg:essp}.\fi} (2) maxflow computation, and (3) repermutation. The process iteratively refine the initial labeling according to the random permutation $\pi^{(i)}$. Each permutation actually corresponds to a set of labelings. Different permutations may lead to different local minima. To avoid being trapped to a poor local minimum, for each iteration of ESSP, we try $K=5$ random permutations in our experiments.

\begin{figure}[htb]
\centering
\includegraphics[width=1\columnwidth]{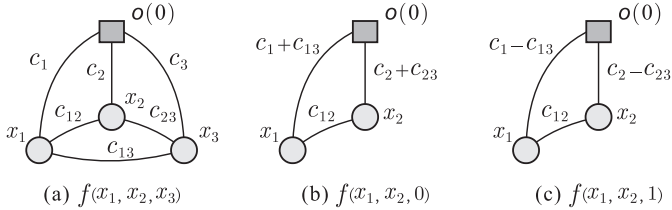}
\caption{Illustration of graph simplification: graph (a) characterizing function $f(x_1,x_2,x_3)$ and its simplified graph (b) characterizing function $f(x_1,x_2,0)$ and simplified graph (c) characterizing function $f(x_1,x_2,1)$.}\label{fig:simp}
\end{figure}

\noindent{\bf Graph Simplification.}~~Some methods, such as QPBO(P) may produce partial labelings \cite{em_ref:kolm07,em_ref:roth07}. Those labeled variables is guaranteed to be globally optimal. In this situations, we can simplify the graph and only focus on minimizing those unlabeled variables. Fig.~\ref{fig:simp} shows how to simplify a graph given a partial labeling.

\noindent{\bf Efficient Implementation.}~~If the modular approximation generates the same $m^{(i)}(\mathbf{x})$ as the last iteration, we need not run the maxflow algorithm. Furthermore, if the difference between $m^{(i)}(\mathbf{x})$ and $m^{(i-1)}(\mathbf{x})$ is only related to a small number of variables, the maxflow computation of last iteration can be efficiently reused \cite{em_ref:kohli07}. For a QPBF of large size, we can improve the efficiency by locally refining the initial labeling using our ESSP algorithm. Specifically, we need just simplify the graph energy by fixing some variables to their initial labels and refine other
variables.

\begin{algorithm}[t] \label{alg:essp}
   \caption{The ESSP Algorithm}
   \begin{small}
\begin{algorithmic}
   \STATE {\bfseries Input:} QPBF $E(\mathbf{x})$ and an initial labeling $\mathbf{x}^{(0)}$
   \STATE {\bfseries Output:} approximate minimizer $\mathbf{x}^*$ of $E(\mathbf{x})$
   \vspace{0.2cm}
   \STATE Construct graph characterization $\mathcal{G}_{E(\mathbf{x})}$ for $E(\mathbf{x})$
   \STATE Supermodular suppression of $\mathcal{G}_{E(\mathbf{x})}$
   \STATE Decompose $\mathcal{G}_{E(\mathbf{x})}$ to $\mathcal{G}_{sub(E(\mathbf{x}))}$ and $\mathcal{G}_{sup(E(\mathbf{x}))}$
   \IF{$isNotEmpty(\mathcal{G}_{sup(E(\mathbf{x}))}$)}
   \STATE $\pi^{(0)} \leftarrow \mathrm{randomPermutation}(\mathbf{x}^{(0)})$
   \REPEAT
   \STATE $m^{(i)}(\mathbf{x}) \leftarrow \mathrm{modApproximation}(sup(E(\mathbf{x})), \pi^{(i-1)})$
   \STATE $\mathbf{x}^{(i)} \leftarrow \arg\min_{\mathbf{x}} \big[sub(E(\mathbf{x})) + m^{(i)}(\mathbf{x}) \big]$
   \STATE $\pi^{(i)} \leftarrow \mathrm{randPermutation}(\mathbf{x}^{(i)})$
   \UNTIL{$\mathrm{isEqual}(\mathbf{x}^{(i)}, \mathbf{x}^{(i-1)})$ is $true$}
   \ELSE
   \STATE Solving an st-mincut on graph $\mathcal{G}_{sub(E(\mathbf{x}))}$
   \ENDIF
   \STATE Converting flipping variables to original forms in $\mathbf{x}^*$
\end{algorithmic}
   \end{small}
\end{algorithm}

\noindent{\bf Optimality and Complexity.}~~General QPBF minimization is NP-hard, thus we can only obtain a suboptimal labeling by ESSP. In each iteration, ESSP actually minimizes an upper bound of the objective function. Given an initial permutation, ESSP iteratively
generates a refined labeling with smaller energy. According to the convergence of SSP \cite{em_ref:nara05}, ESSP is also guaranteed to converge to a local optimum. Besides, automatic submodular-supermodular decomposition also enables us to tell whether the solution is a global optimum or not. Except for maxflow computation, all other steps of ESSP algorithm are of linear complexity. Thus, the complexity of ESSP algorithm is $O(T M)$, where $T$ is the number of iterations and $M$ is the state-of-the-art complexity of maxflow computation on undirected graphs \cite{em_ref:boros02,em_ref:queyranne98,em_ref:onmgc2012}. Compared to the directed graph formulation and QPBO \cite{em_ref:kolm07,em_ref:roth07}, ESSP has half number of vertices in the undirected graph characterization. From our experiments, we find that ESSP usually converges within a small number of iterations when fed by a proper initialization.

\section{Experimental Results}
\label{sec:exp}

We now evaluate the performance of our ESSP algorithm and several state-of-the-art methods, \ie, BP \cite{em_ref:weiss01}, TRW-S \cite{em_ref:kolmogorov06} and QPBO(P,I) \cite{em_ref:kolm07,em_ref:roth07}, on minimizing dense and nonsubmodular energy functions.\footnote{For QPBO, we used the authors' original implementation. For BP and TRW-S, we used the speed-up implementation introduced in \cite{em_ref:kolm06}. All source codes used in our experiments, including the proposed ESSP, are publicly available online.}

\begin{figure}[t]
\centering
\includegraphics[width=0.98\columnwidth]{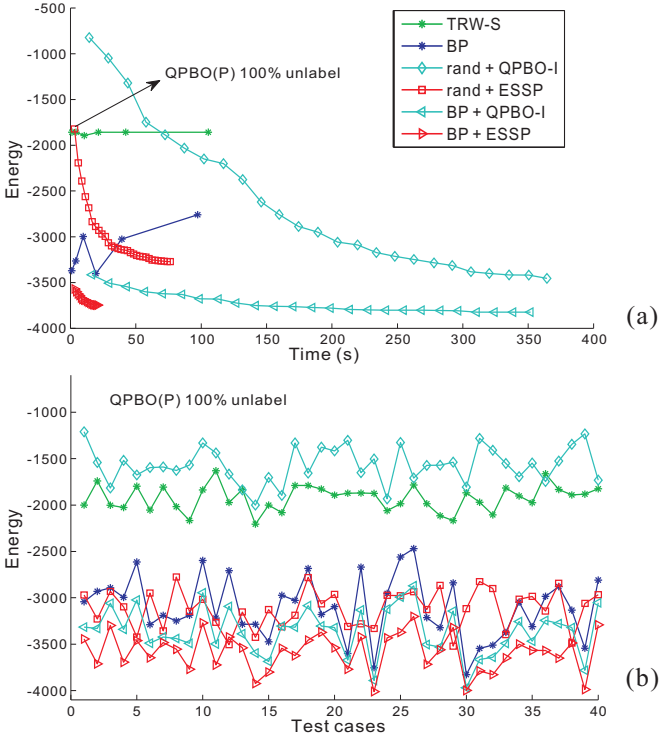}
\caption{Comparison results of six algorithms for dense and nonsubmodular QPBF minimization: (a)~average energy vs. running time curves for $20$ QPBFs of $500$ variables with $C_r=50\%$, $S_r=50\%$ and $U_g=0.1$; (b)~the obtained energies for $40$ random QPBFs under the same configuration. Each algorithm was forced to run at most $60$ seconds in (b).}\label{fig:generalcomp}
\end{figure}

\noindent{\bf Synthetic Comparison.}~~Our first evaluation was based on synthetic QPBFs. We did synthetic comparison due to two reasons. First, unlike the classical energies defined in $4$-connected grid graphs, dense and nonsubmodular QPBFs have only been studied and applied in computer vision recently. Both its potential and difficulty may not necessarily limited in the range of real energies currently available in computer vision. Instead, it would be much more useful if we could figure out what factors of the energy function primarily affect the optimization hardness. Different combinations of these factors define several hardness situations of the problem. We are more interested in finding out the best combination of existing techniques in minimizing dense and nonsubmodular energies at each hardness situation. This will provide a useful guidance for future study and application of such energies in practice. Using synthetic energies, we can easily control the degree of these factors and systematically compare the performance of existing methods, while current real energies may not cover a large spectrum of these factors yet. Second, as stated in \cite{em_ref:meltzer05,em_ref:kolm06}, the performance of different methods for real energies depends on both the power of energy models and the efficacy of energy minimization methods. Using synthetic energies, we can rule out the influence of energy models and purely focus on evaluating the performance of optimization methods only.

\begin{figure*}[t]
\centering
\includegraphics[width=1\textwidth]{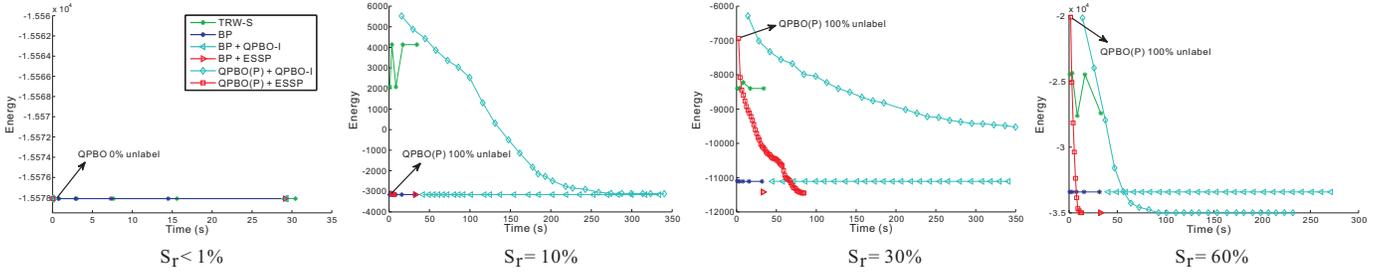}
\caption{The influence of supermodular ratio $S_r$ to the performance of different methods. For each supermodular ratio, we tested $10$ QPBFs with $500$ variables and fixed $C_r=50\%$ and $U_g=0.1$.}\label{fig:supercomp}
\end{figure*}

\begin{figure*}[t]
\centering
\includegraphics[width=1\textwidth]{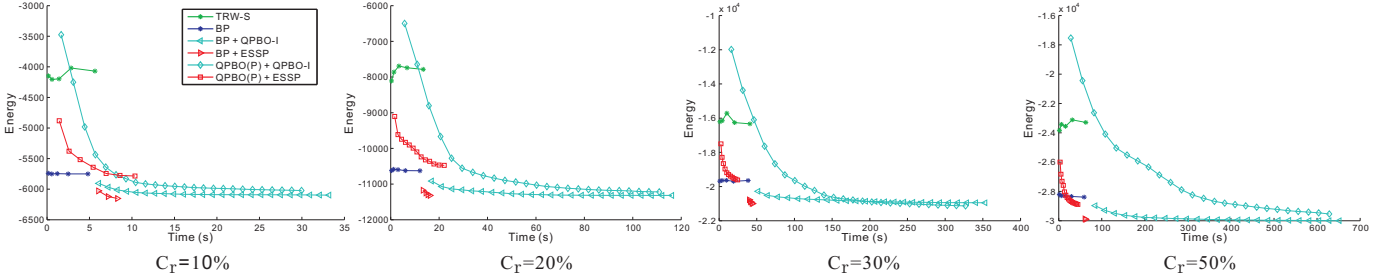}
\caption{The influence of connectivity $C_r$. Each curve was obtained by marginalizing $24 \times 20$ curves.}\label{fig:cr}
\end{figure*}

We generated synthetic binary energy functions with the form of Eq.~(\ref{eq:qpbfstd}), of which the coefficients $\theta_u$ and $\theta_{uv}$ were randomly generated from a uniform distribution $U(0,10)$. Proposition~2 \if0\ref{prop:qpbfgeneral}\fi guarantees that, by this means, we can produce any QPBF with random coefficients. We empirically find three major factors that impact the hardness of the problem: (1)~\emph{connectivity} $C_r=\frac{2 e^\mathcal{V}}{n^2}$ ($n$ is the number of variable nodes, $e^\mathcal{V}$ is the number of variable edges, \ie, edges linking variable nodes), (2)~\emph{supermodular ratio} $S_r=\frac{e^{\mathcal{V}-}}{e^\mathcal{V}}$ ($e^{\mathcal{V}-}$ and $e^{\mathcal{V}+}$ is the number of variable edges of negative and positive capacities, respectively, and $e^\mathcal{V}=e^{\mathcal{V}-}+e^{\mathcal{V}+}$), and (3)~\emph{unary guidance} $U_g=\frac{\mathrm{mean}_u(|c_{uo}|,|c_{u\bar{o}}|)}{\mathrm{mean}_{uv}{|c_{uv}|}}$ ($c_{uo}$, $c_{u\bar{o}}$ and $c_{uv}$ is the capacity of edge $uo$, $u\bar{o}$ and $uv$, respectively).

Fig.~\ref{fig:generalcomp}(a) plots the average energy vs. time curves of six combinations of existing methods in minimizing $20$ random QPBFs with $500$ variables and $C_r=50\%$, $S_r=50\%$ and $U_g=0.1$. The tested methods include BP and TRW-S (both were run for $5000$ iterations), QPBO-I and ESSP with random initialization, QPBO-I and ESSP initialized by BP ($50$ iterations). We do not show the results of QPBO and QPBO-P since under this configuration almost all variables were unlabeled by QPBO(P). We can see that for those optimization methods based on solving LP relaxation, such as TRW-S and QPBO, minimizing dense and nonsubmodular energy functions is extremely difficult. In contrast, other methods such as BP, ESSP and QPBO-I can obtain a relatively lower energy. Note that BP+ESSP clearly outperforms other methods in this test. This is due to the good initialization provided by BP (as compared to other methods) and the fast convergence speed of ESSP. As discussed later, the appealing performance of BP in this test is mainly attributed to the strong unary guidance. Fig~\ref{fig:generalcomp}(b) shows the produced energy of the six solvers for $40$ QPBFs generated using the same configuration. This time, we adjusted the maximal iterations of each solver to make them produce results within $60$ seconds. We can see that using comparable time, QPBO-I with random initialization generated the highest energies. But, using the same random initialization, ESSP obtained much lower energy.

Then, we evaluate the influence of supermodularity to the performance of different methods. Fig.~\ref{fig:supercomp} shows the comparison results. Note that if the supermodular ratio $C_r$ is close to zero, then QPBO(P) is able to label all variables thus we can efficiently obtain a global optimum labeling. Note that although all tested methods produced the same energy to QPBO(P) in the first figure, only QPBO(P) and the algorithms initialized by QPBO(P) are theoretically guaranteed to obtain the global minimum energy.

\begin{figure*}[t]
\centering
\includegraphics[width=1\textwidth]{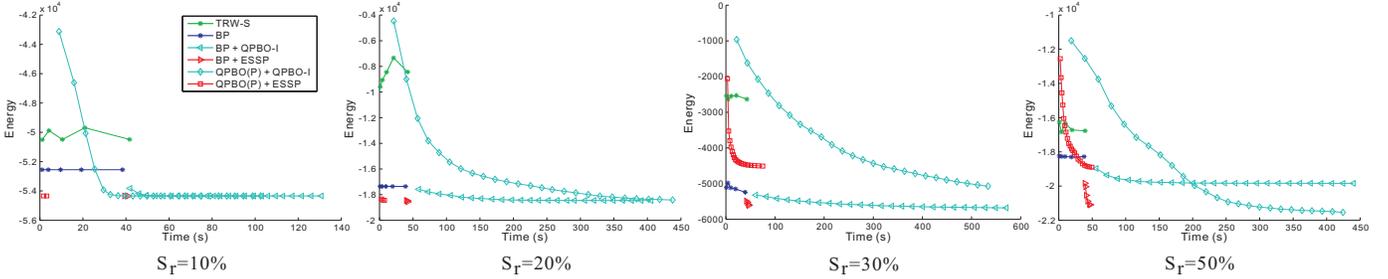}
\caption{The influence of supermodularity ratio $S_r$. Each curve was obtained by marginalizing $24 \times 20$ curves.}\label{fig:sr}
\end{figure*}

\begin{figure*}[t]
\centering
\includegraphics[width=1\textwidth]{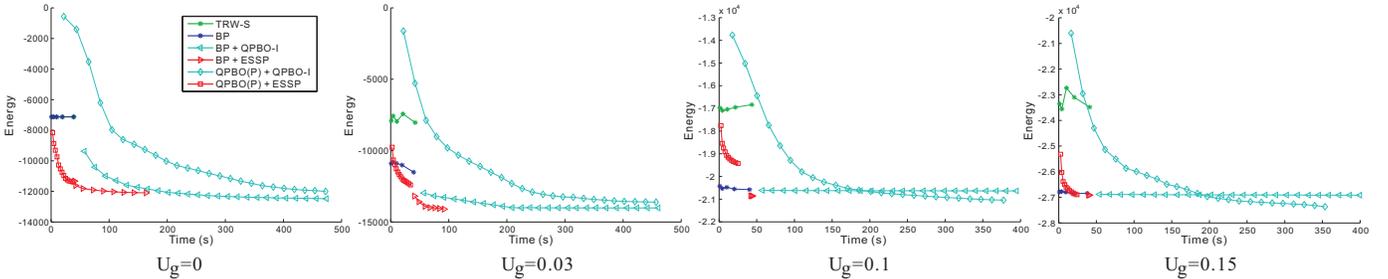}
\caption{The influence of unary guidance $U_g$. Each curve was generated by marginalizing $36 \times 20$ curves.}\label{fig:ug}
\end{figure*}

Then, to make a thorough comparison, we randomly generated $2880$ QPBFs of $600$ variables that covered a large range of connectivity $C_r \in \{0.1,0.2,0.3,0.4,0.5,0.6\}$, supermodularity ratio $S_r \in \{0.1,0.2,0.3,0.4,0.5,0.6\}$ and unary guidance $U_g \in \{0, 0.03, 0.1, 0.15\}$. For each configuration of $C_r$, $S_r$ and $U_g$, we tested $20$ QPBFs. We plotted the performance curves of different methods for each factor value by marginalizing all its energy vs. time curves under this factor value. Fig.~\ref{fig:cr}, \ref{fig:sr} and \ref{fig:ug} show the performance curves of six testing methods under different situations of connectivity, supermodularity and unary guidance.

Since the performance and tradeoff of existing techniques on minimizing sparse energies defined in $4$-connected grid graphs is well-understood in computer vision \cite{em_ref:mrfcomp08,em_ref:tappen03}, as shown in Fig.~\ref{fig:cr}, our study focused on minimizing dense energy functions. We can see that the convergence speed of ESSP clearly outperforms QPBO-I using the same initialization as the connectivity increases. It is worth to note that the initialization of BP is usually better than QPBO(P). This is mainly because for dense and nonsubmodular energies, most variables cannot be certainly labeled by QPBO(P). Thus, in this situation, initialization by QPBO(P) is very close to random initialization.

Fig.~\ref{fig:sr} tells us that lower $S_r$ implies less hardness of the problem. All methods, except for TRW-S, can obtain a comparable low energy for $S_r<20\%$. ESSP and QPBO-I provided by a good initialization can always obtain the lowest energy. But, QPBO-I still needs much more time to converge than ESSP. We also tested the performance of different methods when supermodular ratio $S_r$ is very small. Under this configuration, QPBO(P) is able to label almost all variables, thus we can certainly obtain a global optimum labeling. It is also worth to note that for energy functions with low $S_r$, we can use the partial labeling generated by QPBO(P) to simplify the graph of ESSP and run ESSP to seek a suboptimal labeling for other variables. In our tests, however, we found that this is equivalent to initializing ESSP by the output of QPBO(P). The optimality of the partial labeling generated by QPBO(P) can also be preserved by ESSP.

\begin{table*} [t]
\centering {\small
\caption{Recommended methods for various QPBF minimization tasks}
\renewcommand{\multirowsetup}{\centering}
\begin{tabular}{|c|c|cc|}
\hline \begin{minipage}[t]{4.5cm}{Sparser connectivity {\small $C_r<1\%$}} \end{minipage} & \multicolumn{3}{c|}{Denser connectivity {\small $C_r>5\%$}} \\
\hline \raisebox{-0.5ex}[0pt]{\multirow{4}{4cm}{Refer to \cite{em_ref:mrfcomp08,em_ref:tappen03}}} & \begin{minipage}[t]{4.5cm}{Lower supermodularity {\small $S_r\approx 0$}} \end{minipage} & \multicolumn{2}{c|}{\begin{minipage}[t]{5.2cm}{Higher supermodularity {\small $S_r > 10\%$}} \end{minipage}} \\
\cline{2-4} & \multirow{3}{3cm}{QPBO(P)+ESSP} & \multirow{1}{3cm}{{\small $U_g\approx 0$}} \vline& {\small $U_g > 5\%$} \\
\cline{3-4} & & \multirow{2}{3cm}{rand+ESSP+I} \vline& \multirow{2}{2.2cm}{BP+ESSP} \\
& & \multirow{1}{3cm}{ } \vline& \\
\hline
\end{tabular} \label{tab:recmethod}}
\end{table*}

From Fig.~\ref{fig:ug}, we can see that if the energy function provides weak unary guidance, \ie, $U_g \approx 0$, BP and TRW-S equally perform worse than ESSP and QPBO-I with random initialization. However, when the unary guidance is strong enough, BP can usually obtain a relatively low energy. The unary guidance actually reflects the consistency of unary and pairwise potentials. For an energy function in computer vision, $U_g=0$ means that the unary terms caused by data likelihood fully contradict the unary terms contributed by pairwise potentials, thus may apparently increase the hardness of the problem.

\noindent{\bf Image Restoration.}~~We also tested different methods on binary image restoration based on the KAIST Chinese character database. Specifically, we $50$ images to train a dense pairwise prior function for a particular character $\psi(\mathbf{x}) = \sum_{(u,v) \in \mathcal{N}^2} f_{uv} (x_u, x_v)$, where $\mathcal{N}$ is the set of pixels, $f_{uv} (x_u, x_v)$ is the appearance frequency of labeling $(x_u, x_v)$ for pixel $u$ and $v$ in the training data. As we consider every pixel pair and the desired labeling can be either $00$, $01$, $10$ or $11$, the prior function is certainly dense and nonsubmodular. In our experiments, the average $C_r$ is above $70\%$, and average $S_r$ is about $20\%$. To avoid over-fitting, we omit those labelings whose appearance frequency in the training data is less than $10\%$. Using the trained dense prior function $\psi(\mathbf{x})$, we restore a noisy image by minimizing the energy function $E(\mathbf{x}) = \alpha \sum_u \frac{-1}{1+|Y_u-x_u|}  + \beta \psi(\mathbf{x})$, where $Y_u$ is the appearance of noisy image for pixel $u$, $x_u
$ is its label. Fig.~\ref{fig:charrestore} shows the restoration results for one character.

\begin{figure}[t]
\centering
\includegraphics[width=1\columnwidth]{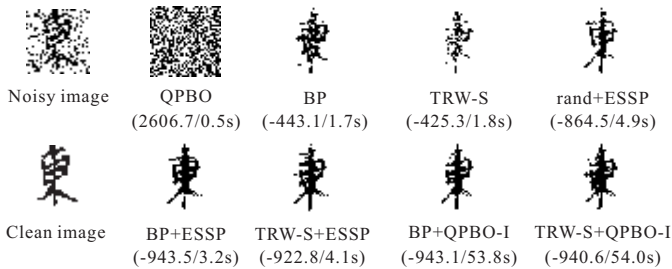}
\caption{Binary image restoration of different methods. The energy and time for each method are shown in the format of `energy/time(s)'. The lower bound of the energy function is $-980.0$.} \label{fig:charrestore}
\end{figure}

\noindent{\bf Discussion.}~~In our experiments, we found that TRW-S consistently performs worse than other methods. This is in contrast to their performances for sparse energy functions \cite{em_ref:meltzer05,em_ref:mrfcomp08}. The difficulty of TRW-S on solving dense energy functions was also found and analyzed in \cite{em_ref:kolm06}. It seems that dense connectivity has much more negative influence to the performance of LP relaxation based methods, \eg, QPBO and TRW-S, than other techniques, \eg, BP, QPBO-I and ESSP.

The experiments show that ESSP has much faster convergence speed than QPBO-I given the same initialization. Note that we did not make any speed optimization in our implementation. As shown in Fig.~\ref{fig:cr}, \ref{fig:sr} and \ref{fig:ug}, QPBO-I is able to produce a lower energy than ESSP if given enough long time. Both the performance of ESSP and QPBO-I depends on the goodness of initialization. For ESSP, bad initialization leads to a poor local minimum, while good initialization leads to a much lower energy (see Fig.~\ref{fig:generalcomp}). In contrast, for QPBO-I, bad initialization means much slower convergence speed than good initialization.

Our comparative study helps to understand the hardness of dense and nonsubmodular energy minimization. More importantly, from our study, we can make some promising recommendations of existing methods in different situations that performs the best to minimize such challenging energies, as summarized in Table~\ref{tab:recmethod}. Generally, no existing methods are able to handle all situations. But, the proposed ESSP algorithm can always be used to efficiently improve the labeling obtained by existing methods for dense and nonsubmodular binary MRF energy functions.  Note that, $U_g\approx 0$ corresponds to the most difficult case, since both ESSP and QPBO-I cannot be initialized properly.

\section{Conclusions}
\label{sec:conc}

We have proposed a new algorithm, namely ESSP, to minimize dense and nonsubmodular energy functions. Such kind of energies has recently shown great potentials in computer vision. Our approach is based on an undirected graph characterization of QPBFs, which enables us to extend the classical submodular-supermodular procedure \cite{em_ref:nara05} to a general solver for generic binary labeling problems. Experiments show that for dense and nonsubmodular energy functions, our ESSP algorithm can usually improve the results of existing methods with reasonable time.

We have also provided a thorough comparative study on minimizing dense and nonsubmodular QPBFs by existing techniques. We empirically find out three important factors, \ie, connectivity, supermodularity and unary guidance, that closely relate to the hardness of the problem. Based on our study, we finally make several reasonable recommendations of combining existing methods in different situations to minimize such challenging energies. We believe our study presents a positive guidance for future modeling and applications of general energy functions in computer vision.


There are several open questions about the ESSP algorithm. Like SSP, ESSP is also an iterative refinement process. It would be desirable if we could analyze the bound on the maximal number of iterations. From our experiments, we observe that ESSP usually converges to a local minimum after a small number of iterations. Besides, combining ESSP with other recent optimization methods, such as \cite{em_ref:onmgc2012,em_ref:komo08}, is an interesting direction for future work.

\ifCLASSOPTIONcompsoc
  \section*{Acknowledgments}
\else
  \section*{Acknowledgment}
\fi

This work is supported by the Program for New Century Excellent Talents in University (NCET-11-0365), National Nature Science Foundation of China (61100121), and National Science and Technology Support Project (2013BAK01B01).

\ifCLASSOPTIONcaptionsoff
  \newpage
\fi

\bibliographystyle{IEEEtran}
\bibliography{IEEEabrv,essp}


%







\end{document}